\documentclass[a4paper,10pt,onecolumn]{article}

\usepackage{preprint}

\title{Evolving Markov Chains: Unsupervised Mode Discovery and Recognition from Data Streams}

\author[1]{Kutalmış Coşkun}
\author[2]{Borahan Tümer}
\author[1]{Bjarne C. Hiller}
\author[1]{Martin Becker}
\affil[1]{University of Rostock, Germany}
\affil[2]{Marmara University, Türkiye}

\date{}

\begin{document}

\maketitle
 
\begin{abstract}
    Markov chains are simple yet powerful mathematical structures to model temporally dependent processes.
    They generally assume stationary data, i.e., fixed transition probabilities between observations/states.
    However, live, real-world processes, like in the context of activity tracking, biological time series, or industrial monitoring, often switch behavior over time.
    Such behavior switches can be modeled as transitions between higher-level \emph{modes} (e.g., running, walking, etc.).
    Yet all modes are usually not previously known, often exhibit vastly differing transition probabilities, and can switch unpredictably.
    Thus, to track behavior changes of live, real-world processes, this study proposes an online and efficient method to construct Evolving Markov chains (EMCs).
    EMCs adaptively track transition probabilities, automatically discover modes, and detect mode switches in an online manner.
    In contrast to previous work, EMCs are of arbitrary order, the proposed update scheme does not rely on tracking windows, only updates the relevant region of the probability tensor, and enjoys geometric convergence of the expected estimates.
    Our evaluation of synthetic data and real-world applications on human activity recognition, electric motor condition monitoring, and eye-state recognition from electroencephalography (EEG) measurements illustrates the versatility of the approach and points to the potential of EMCs to efficiently track, model, and understand live, real-world processes.
\end{abstract}

\section{Introduction}
\label{sec:introduction}

% MOTIVATION
Live real-world systems such as healthcare monitoring platforms, industrial automation processes and environmental monitoring systems generate vast amounts of real-time data that require continuous tracking for management, optimization, and decision-making \cite{khanra_BigData_2020,wang_BigData_2022}.
These systems are generally non-stationary, meaning their behavior changes over time due to various influencing factors.
This dynamic nature makes identifying and adapting to distinct behavioral states --- which are commonly referred to as \emph{modes} \cite{buede_EngineeringDesign_2009} and illustrated in \cref{fig:ga} (a) --- a challenging but equally useful task with immense practical value.
For instance, a healthcare monitoring system might operate differently depending on whether a patient is in the \enquote{mode} resting, exercising, or experiencing a medical emergency \cite{alsadoon_ArchitecturalFramework_2024}.
Similarly, a smart manufacturing plant might optimize maintenance costs by responding to known or unknown anomalies occurring in various stages of production, where transitions from normal to anomalous modes could be caused by equipment breakdowns or material defects \cite{aydemir_AnomalyMonitoring_2020}.
In this context, each mode represents a distinct type of behavior associated with characteristic observations exhibited by the system under specific conditions.
Discovering and recognizing these modes are central tasks for devising adaptive and robust systems that can dynamically adjust their operations in real-time, which not only can help maintain optimal performance but also allows preemptively addressing potential issues.

% PROBLEM SETTING
In this context, the need for real-time feedback and the constant influx of data from real-world systems necessitate an online approach to recognize modes
from streaming data.
Thus, optimally, a corresponding method will process the most recent observation from the system and predict the currently active mode before the next observations, as illustrated in the mode switching diagram in \cref{fig:ga} (d).
In real-world settings, this task is particularly challenging, as prior information about the number of modes and their behavior is often not available. 
That is,
\begin{inparaenum}[(i)]
    \item in general, for individual observations, labels regarding the active mode are unavailable (an example output sequence is shown in \cref{fig:ga} (b)), 
    \item there are no descriptions or models available for the modes,
    \item the number of modes is unknown,
    \item the transitions from one mode to another happen at unknown times,
    \item there is no feedback whether the predicted active mode is correct at runtime.
\end{inparaenum}

This problem setting is akin to clustering.
However, rather than associating individual observations with \emph{clusters} as in the standard \emph{stream clustering} \cite{zubaroglu_DataStream_2021} setting, here, contiguous time intervals in the data stream are identified (referred to as \emph{regimes}) which are in turn associated with \emph{modes}.
This allows to define modes as behavioral states and associate individual observations with modes by their temporal dependencies within a regime.
A simple example are observations of normal and high heart rate measurements in resting and exercise modes.
During a regime in exercise mode, the likelihood of observing a high heart rate measurement following another high heart rate measurement is significantly increased due to continuous physical exertion.
Conversely, during a regime in resting mode, a high heart rate is more likely to return to normal rather than remain elevated.
This difference in the temporal dependencies between observations shows how different modes can exhibit distinct patterns of behavior.

% APPROACH
To discover and recognize the underlying modes of a dynamic system, this study proposes \acp{propAlg}.
\acp{propAlg} is an online method that constructs a dynamic $k$th order Markov chain and efficiently updates it after each observation from a stream of categorical values.
Changes in the continuously updated Markov chain are tracked to detect and recognize mode switches.
For each discovered mode, the parameters of the Markov chain are stored in memory to detect recurrences.
An overview of \ac{propAlg} is shown in \cref{fig:ga}.
%
% CONTRIBUTION
The contribution of this study is threefold:
\begin{itemize}
    \item An efficient --- $O(mk)$ for each observation where $m$ is the cardinality of the state space --- and provably asymptotically convergent update mechanism for online estimation of $k$th order conditional probabilities.
    \item An algorithm to detect drifts and recognize the underlying modes of the target system from data streams by utilizing estimated conditional probabilities.
    \item In-depth illustration of the proposed method through synthetic and real-world applications from a variety of domains, such as
    \begin{inparaenum}
        \item human activity recognition,
        \item electric motor condition monitoring,
        \item human eye-state tracking through \ac{EEG} measurements.
    \end{inparaenum}
\end{itemize}

\begin{figure}[htbp]
    \centering
    \includegraphics[width=\linewidth]{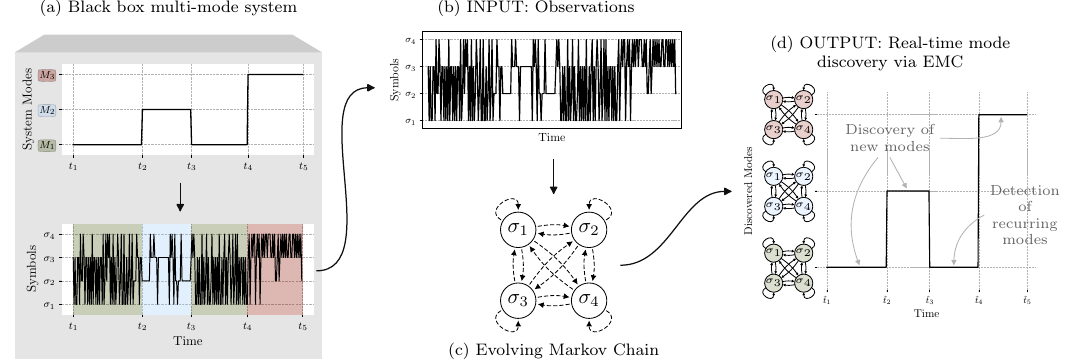}
    \caption{
        (a) The black box process has unknown underlying modes ($M_1$, $M_2$ and $M_3$) and it switches from one mode to another at arbitrary times.
        In this case, the mode switching sequence is $M_1\to M_2\to M_1\to M_3$ (top figure).
        The active mode at a given time determines how the symbol sequence is generated (bottom figure).
        (b) The sequential data generated by the black box process contains neither information about the underlying modes nor the switching points.
        (c) \ac{propAlg} processes the sequence in an online manner, and updates the state transition probabilities accordingly to discover the underlying modes.
        (d) The switching diagram of discovered modes shows which mode is active at each time point.
        Discovered modes are represented with Markov chains with different transition probabilities, indicated with different colors.
    }
    \label{fig:ga}
\end{figure}

% PAPER STRUCTURE
The paper is organized as follows: In \Cref{sec:related_work}, we review related work, and in \Cref{sec:preliminaries}, we provide background information.
We define the \ac{OMDR} problem in \Cref{sec:problem} and introduce the proposed method in \Cref{sec:methodology}.
Experimental evaluation is presented in \Cref{sec:experimental_evaluation}.
Finally, in \Cref{sec:discussion_and_future_work}, we discuss results and outline future work, concluding with \Cref{sec:conclusion}.

\section{Related Work}
\label{sec:related_work}

We discuss the related work under the main topics of 
\begin{inparaenum}
    \item \emph{data stream clustering}, due to the connection to the mode discovery/recognition problem and 
    \item \emph{leveraging the temporal structure}, as it is related to the main contribution of this study.
\end{inparaenum}

\subsection{Data Stream Clustering}

The goal in data stream clustering is to maintain a finite set of categories (i.e., clusters) that accurately describe the sequence observed so far, while adhering to memory and time restrictions \cite{silva_DataStream_2013}.
Although the goal is similar, stream clustering poses a distinct set of challenges compared to traditional clustering \cite{zubaroglu_DataStream_2021,aggarwal_SurveyStream_2018,mousavi_DataStream_2015}, such as 
\begin{inparaenum}[(i)]
    \item data instances can only be processed once and in a certain order (i.e., as they arrive),
    \item processing shall be done before the next instance arrives, and
    \item it is not feasible to store the whole stream, only a form of synopsis regarding the data is stored.
\end{inparaenum}
These points are also valid for the \ac{OMDR} problem (see \Cref{subsec:online_estimation_of_cps}), and therefore apply to the proposed method, \ac{propAlg}, as well.

A common recipe for data stream clustering algorithms, which \ac{propAlg} also mostly follows, involves 
\begin{inparaenum}[(i)]
    \item an online method to generate summaries that represent different characteristics of the stream,
    \item a distance (or similarity) function that operates on the instance/cluster or cluster/cluster tuples, and
    \item a procedure to assign instances to clusters or optionally start a new cluster.
\end{inparaenum}
Most of the methods employ a two phase approach to utilize these components \cite{zubaroglu_DataStream_2021, silva_DataStream_2013,ackermann_StreamKMClustering_2012,yin_ImprovedClustering_2018}, namely in the \emph{online phase}, the synopsis of the data is extracted and in the \emph{offline phase} the final clustering is done via standard clustering methods periodically or upon request from the user.
On the contrary, the method we propose is \emph{truly online}, that is, both the summary generation (constructing and updating the Markov chain) and the sample to cluster assignment are realized online.
Other fully online approaches utilize drift detection mechanisms \cite{puschmann_AdaptiveClustering_2017} or constant measurement of clustering quality \cite{andradesilva_EvolutionaryAlgorithm_2017} to make standard $k$-means applicable to the streaming scenario.
There are also density based \cite{hyde_FullyOnline_2017} and distance based \cite{zhang_DBIECManEvolving_2017} fully online methods as well as ones that work with categorical data \cite{aggarwal_ClusteringMassive_2010}.
While these approaches primarily focus on finding and maintaining clusters that accurately represent the data, they all lack a capability that is particularly important for the \ac{OMDR} problem we address in this paper.
Specifically, they are reliant on instance features to find clusters and do not effectively utilize the temporal information that is inherently present in data streams.
To address this gap, \ac{propAlg} performs stream clustering by constructing a dynamic Markov chain that intrinsically extracts the temporal structure of the data in the form of conditional probabilities and does not rely on instance features.

\subsection{Leveraging the Temporal Structure}

In the context of understanding the temporal structure of sequential data, Markov models stand out due to their simple yet powerful ability to model state transitions probabilistically~\cite{ching_MarkovChains_2013}.
The simplest offline way to build the stochastic matrix of a Markov chain from a sequence of observations is to calculate the transition frequencies and normalize them by the total number of transitions originating from the corresponding state \cite{gagniuc_MarkovChains_2017}.
There are also offline approaches that utilize genetic algorithms \cite{wang_EvolutionaryMarkov_2022} and Bayesian statistics \cite{mizutani_ImprovingEstimation_2017}.
Another prominent offline method that builds a \emph{variable} order Markov chain \cite{schulz_FastAdaptive_2008} has recently been used for mode discovery and recognition in industrial domains \cite{surmeli_MultivariateTime_2020,coskun_SyntacticPattern_2022}; and although it is not an online method, we provide comparisons of \ac{propAlg} with this approach in experimental evaluation part of this study. 

On the other hand, if the measurements/observations are obtained in real-time or the target process has a non-stationary nature, a mechanism to \emph{adapt} the changing conditions is needed.
A natural way to achieve this is to use the most recent transitions through a landmark/sliding window \cite{ren_AnomalyDetection_2017,tilahun_CooperativeMultiagent_2017}.
An alternative approach is to maintain mean transition frequencies \cite{kolmanovsky_StochasticOptimal_2009}.
This is similar to how our method (\ac{propAlg}) estimates conditionals probabilities from a mathematical point of view, i.e., it results in exponential forgetting of older transitions.
However, the way \ac{propAlg} derives transition probabilities has significant computational advantages: 
\begin{inparaenum}[(i)]
    \item \ac{propAlg} only updates the relevant region of the transition probability tensor, which costs $O(mk)$ (where $m$ is the cardinality of the state set and $k$ is the order of the conditional probabilities) compared to the full update that costs $O(m^{k+1})$ as proposed in \cite{kolmanovsky_StochasticOptimal_2009}, and
    \item \ac{propAlg} directly operates on the transition probabilities instead of frequencies, which removes the need for maintaining an additional value for normalization and the division operation.
\end{inparaenum}
In addition, \ac{propAlg}'s update scheme can estimate $k$th order dependencies in contrast to only first order, and we provide a mathematical proof on the convergence of expected transition probabilities.

In a related study \cite{aslanci_DetectionRegime_2017}, \ac{SCD} algorithm is proposed to find change points in sequences by estimating the conditional probabilities.
\ac{propAlg} builds on top of \ac{SCD} by refining the estimation mechanism and utilizing conditional probabilities for the \ac{OMDR} problem.
A comparison of \ac{propAlg} and \ac{SCD} on change detection task is provided in the experimental evaluation part.

Another relevant work that focuses on learning the temporal structure of sequential data is \ac{TRACDS} \cite{hahsler_TemporalStructure_2011}.
Although the motivation of \cite{hahsler_TemporalStructure_2011} is similar to ours, \ac{TRACDS} learns the temporal structure of the clusters, not the data points.
Therefore, it misses the temporal dependencies that exist in the sequence that could characterize a mode, rather it is proposed as an extension to an existing stream clustering method.

Overall, while there are methods that share the same motivation of utilizing the temporal structure to cluster sequential data, given the requirements of real-time processing, there is a need for computationally efficient methods that are applicable for tasks like \ac{OMDR}.

\section{Preliminaries}
\label{sec:preliminaries}

This section provides information regarding the definitions and terminology referred to throughout the paper.

\subsection{Non-Stationary Systems}
\label{subsec:nss}

In the general sense, \emph{stationarity} is defined as a property of a signal (or a system through a signal obtained from it) that indicates invariance regarding time \cite{box_TimeSeries_2016}.
Non-stationarity is simply the absence of this property due to the intrinsic \emph{evolving} or \emph{drifting} nature of the system \cite{ditzler_LearningNonstationary_2015}.
There could be different ways in which a system might be non-stationary, such as the drift could be sudden, gradual, incremental and recurring \cite{lu_LearningConcept_2018}.
Considering the \ac{OMDR} problem (formally defined in \cref{sec:problem}), we focus on mode switches that are sudden and possibly recurring, in which case, a \emph{change point} or \emph{switching point} defines a point in time that separates two different behavioral segments (i.e., regimes) that are associated with different modes.
An example of such non-stationary behavior is shown in \cref{fig:ga} (a).

\subsection{Online Estimation of Multinomial Distributions}
\label{subsec:prelim_online}

For online estimation of the probability distribution $P_{X}$ from the realizations of random variable $X$, \citet{oommen_StochasticLearningbased_2006} presented \ac{SLWE} that uses the update equation in \cref{eq:slwe_multinomial_update_rule}.
Here, $\hat{p}_{\sigma}[n]$ represents the estimate at time $n$ of the probability of observing realization $\sigma$ and $0<\lambda<1$ is the learning coefficient.
\begin{equation}
    \label{eq:slwe_multinomial_update_rule}
    \hat{p}_{\sigma_{i}}[n] \gets
    \begin{cases*}
        \hat{p}_{\sigma_{i}}[n-1]+(1-\lambda)\sum\limits_{j \neq i}\hat{p}_{\sigma_{j}}[n-1],&if $X_{n}=\sigma_{i}$ \\
        \lambda \hat{p}_{\sigma_{i}}[n-1],&if $X_{n} \neq \sigma_{i}$
   \end{cases*}
\end{equation}

While the probability distribution $P_{X}$ estimated by \ac{SLWE} provides insight into certain properties of $X$, it does not capture the temporal structure inherent in the sequence of realizations of $X$.
Since $P_{X}$ is merely a non-conditional probability distribution, it lacks the dependency information necessary to describe the temporal relationships between successive realizations of $X$.
However, the update scheme used in \ac{SLWE} forms the basis of how we predict the $k$th order conditional probabilities in this study, which is described in detail in \cref{subsec:online_estimation_of_cps}.

\section{Problem Statement: Online Mode Discovery and Recognition}
\label{sec:problem}

We refer to the problem studied in this article as Online Mode Discovery and Recognition (\ac{OMDR}).
A mode is defined as a distinct operational capability of a system \cite{buede_EngineeringDesign_2009}.
While a system might have multiple fully operational modes (e.g., low battery/charging modes of a laptop), there might also be partially operational modes (e.g., maintenance mode of an elevator system) and unwanted failure modes (e.g., sensor failure in autonomous vehicles).
Thus, in \emph{mode discovery}, the goal is to extract such distinct behavior types solely from the output of the system (e.g., measurements obtained through a set of sensors) without any label regarding the underlying modes.
In \emph{mode recognition}, on the other hand, the goal is to identify the discovered mode based on the previously seen ones, which necessitates a notion of memory that stores a model for each mode.
Finally, in \ac{OMDR}, discovery and recognition are done online, that is, the observations from the system are processed one by one and a prediction is done regarding the active mode at each time step.
In \cref{def:omdr}, the \ac{OMDR} problem is formally defined.

\begin{definition}[Online Mode Discovery and Recognition]
    \label{def:omdr}
    Let $X=\langle x_{1},x_{2},\dots,x_{t},\dots \rangle$ be an infinite sequence of uni- or multivariate data points where $t\in\mathbb{N}^{+}$ represents the time index.
    Observations are from a finite set of symbols $\sigma \in \Sigma$, also called alphabet.
    Also, let $\mathcal{M}_t$ be the set of modes that are known at time $t$.
    The goal in \ac{OMDR} is to assign a mode $M\in\mathcal{M}_t$ to $x_t$ as $x_t$ becomes available, where $\mathcal{M}_t$ is the set of modes that are known at time $t$ and $\mu(x_t) = M$ denotes the assigned mode for $x_t$.
    If $x_t$ can not be assigned to any mode in $\mathcal{M}_t$ a new mode $\bar{M}$ is created and $\mathcal{M}_{t+1} = \mathcal{M}_{t} \cup \{\bar{M}\}$.
    Regimes are disjunct time intervals from time point $t_{s}$ to $t_e$, where all observations $x_t$ are assigned to the same mode $M$: $R_{s:e}^M = \{ x_t ~|~  t_s \leq t \leq t_e \wedge \mu(x_t) = M\}$.
    Once a data point $x_t$ is assigned to a mode $M$ the assignment is final and cannot be changed based on future data or observations.
\end{definition}

\section{Evolving Markov Chains (EMCs)}
\label{sec:methodology}

This section introduces the proposed method.

\subsection{Online Estimation of Conditional Probabilities}
\label{subsec:online_estimation_of_cps}

At the heart of our evolving approach lies the online estimation of $k$th order conditional probabilities from a data stream.
These probabilities form the stochastic tensor that defines a $k$th order Markov chain, which is essentially the model that we learn from data to represent a mode (\cref{fig:ga} (d)).

As previously mentioned, \ac{SLWE} \cite{oommen_StochasticLearningbased_2006} applies incremental updates to the probabilities of possible observations (symbols) with \cref{eq:slwe_multinomial_update_rule}.
However, symbol probabilities alone do not capture the temporal structure of the sequence.
Therefore, we extend \cref{eq:slwe_multinomial_update_rule} so that the $k$th order conditional probabilities are estimated.
Namely, let $\vec{P}$ be the $(k+1)$ dimensional non-stationary stochastic tensor of $k$th order conditional probabilities.
In order to maintain an always up-to-date estimate of $\vec{P}$, we update the corresponding \ac{CPD} $\vec{\hat{P}}[n]$ after each observation at time point $n$ with the update rule given in \cref{eq:ehmc_update_rule},
\begin{gather}
    \label{eq:ehmc_update_rule}
    \hat{p}[n] \gets
    \begin{cases*}
        \lambda \hat{p}[n-1]+(1-\lambda) & if $C_{1}$ \\
        \lambda \hat{p}[n-1] & if $C_{2}$ \\
        \hat{p}[n-1] & if $C_{3}$ \\
   \end{cases*} \\[6pt]
   \begin{aligned}
       C_{1}&: \langle X_{n-k}, \dots, X_{n} \rangle = \langle s_{1}, \dots, s_{k+1} \rangle \\
       C_{2}&:(\langle X_{n-k}, \dots, X_{n-1} \rangle = \langle s_{1}, \dots, s_{k} \rangle) \land (X_{n} \neq s_{k+1}) \\
       C_{3} &: \langle X_{n-k}, \dots, X_{n-1} \rangle \neq \langle s_{1}, \dots, s_{k} \rangle
   \end{aligned}\nonumber
\end{gather}
where
$\hat{p}[n]$ stands for the estimated probability $\hat{p}_{(s_{k+1} \given s_{1} \dots s_{k})}$ at time point $n$.
$\hat{p}[n]$ is a component of $\hat{\vec{P}}[n]$ and represents the \enquote{$k$th order} conditional probability of observing the symbol $s_{k+1}$ after the sequence $\langle s_{1},\dots,s_{k} \rangle$.
Here, $s_{i}$ is an observation of a symbol from a finite alphabet $\Sigma$ with $|\Sigma|=m$.
Also, $0<\lambda<1$ represents a learning coefficient.
The case $C_1$ corresponds to the case in which the sequence of past $k+1$ observations $\langle X_{n-k}, \dots, X_{n} \rangle$ matches target $s_{k+1}$ and condition $\langle s_{1},\dots,s_{k} \rangle$ of the probability $\hat{p}[n]$ being updated.
The corresponding probability is decreased by multiplying by $\lambda$ and increased by adding $(1-\lambda)$.
In case $C_2$, the most recent observation $X_n$ does not match $s_{k+1}$. 
Therefore the corresponding probability is decreased by multiplying by $\lambda$.
In case $C_3$, $\hat{p}[n]$ is unchanged since the observations and the condition do not match.

\cref{eq:ehmc_update_rule} ensures that after each update all probabilities conditioned on the same event add up to $1$, i.e., $\sum_{s \in \Sigma} p(s \given X) = 1$, where $X$ can be a single event or a sequence of events according to the Markov order.
This way, the resulting Markov chain always remains up-to-date with only $|\Sigma|=m$ update operations per observation.

In \cref{thm:arbitrary_order_convergence}, we show the weak convergence of the estimate $\hat{\vec{P}}$ obtained with the update rule given in \cref{eq:ehmc_update_rule} to the real probabilities $\vec{P}$.

\begin{theorem}
    \label{thm:arbitrary_order_convergence}
    The estimate $\hat{\vec{P}}[n]$ at time $n$ (i.e., after the $n$th observation) calculated using \cref{eq:ehmc_update_rule}, weakly converges to the actual $(k+1)$ dimensional stochastic tensor $\vec{P}$ of the corresponding $k$th order ergodic Markov chain, 
    i.e.,
    $\mathbb{E}(\hat{\vec{P}}[\infty])=\vec{P}$.
\end{theorem}

\begin{proof}
    To calculate the expected value, we write \cref{eq:ehmc_update_rule} with probabilities assigned to each condition as given in \cref{eq:ehmc_update_rule_with_probs},
    \begin{equation}
        \label{eq:ehmc_update_rule_with_probs}
        \hat{p}_{n} \gets
        \begin{cases*}
            \lambda \hat{p}_{n-1}+(1-\lambda) & with prob. $AB$   \\
            \lambda \hat{p}_{n-1} & with prob. $A(1-B)$  \\
            \hat{p}_{n-1} & with prob. $1-A$
       \end{cases*}
    \end{equation}
    where we use $A$ for $p_{(s_{1}, \dots, s_{k})}$, $B$ for $p_{(s_{k+1} \given s_{1}, \dots, s_{k})}$ and $\hat{p}_{n}$ for $\hat{p}_{(s_{k+1} \given s_{1}, \dots, s_{k})}[n]$ for simplicity.
    
    The conditional expectation of $\hat{p}_{n}$
    given the estimated stochastic tensor from the previous time point $\hat{\vec{P}}[n-1]$ is written as \cite{oommen_StochasticLearningbased_2006}:
    \begin{align}
		\begin{split}
			\mathbb{E}(\hat{p}_{n} \given \hat{\vec{P}}[n-1]) %
            &= AB(\lambda \hat{p}_{n-1}+1-\lambda) \\
            &\quad + A(1-B)\lambda\hat{p}_{n-1} + (1-A)\hat{p}_{n-1} \\
            &= (1-\lambda)(AB-A\hat{p}_{n-1})+\hat{p}_{n-1}
		\end{split}
    \end{align}
    Taking the expectation of both sides one more time, we get:
    \begin{equation}
        \mathbb{E}(\hat{p}_{n}) = (1-\lambda)[AB-A\mathbb{E}(\hat{p}_{n-1})]+\mathbb{E}(\hat{p}_{n-1})
    \end{equation}
    As $n\to\infty$, both $\mathbb{E}(\hat{p}_{n})$ and $\mathbb{E}(\hat{p}_{n-1})$ approach $\mathbb{E}(\hat{p}_{\infty})$:
    \begin{align}
		\begin{split}
            \mathbb{E}(\hat{p}_{\infty}) &= (1-\lambda)[AB-A\mathbb{E}(\hat{p}_{\infty})]+\mathbb{E}(\hat{p}_{\infty}) \\
            0 &= (1-\lambda)[AB-A\mathbb{E}(\hat{p}_{\infty})]
        \end{split}
    \end{align}
    For which to hold, either $\lambda=1$ or $\mathbb{E}(\hat{p}_{\infty})=B$.
    Since $0<\lambda<1$ by definition, we get:
    \begin{equation}
        \mathbb{E}(\hat{p}_{\infty}) = B
    \end{equation}
    Returning to the original notation, we get:
    \begin{equation}
        \mathbb{E}(\hat{p}_{(s_{k+1} \given s_{1}, \dots, s_{k})}[\infty]) = p_{(s_{k+1} \given s_{1}, \dots, s_{k})}
    \end{equation}
    This means that the expected estimate is asymptotically equal to the true value and since this is true for all components of $\hat{\vec{P}}[n]$, we get $\mathbb{E}(\hat{\vec{P}}[\infty]) = \vec{P}$, which completes the proof.
\end{proof}

Although the asymptotic convergence shown in the proof of \cref{thm:arbitrary_order_convergence} may seem impractical for estimating the stochastic tensor of a $k$th order Markov chain, the convergence occurs in geometric manner \cite{oommen_StochasticLearningbased_2006,narendra_LearningAutomata_2012}, making it possible to learn regimes in a non-stationary sequence as defined in \cref{subsec:nss}.

Next, we describe entropy regulation, an auxiliary update mechanism that complements \cref{eq:ehmc_update_rule} and addresses cases where non-updated components of $\vec{P}$ pose potential issues.

\subsection{Entropy Regulation}

The update scheme in \cref{eq:ehmc_update_rule} does not modify the probability estimate $\hat{p}_{(s_{k+1} \given s_{1}, \dots, s_{k})}$ if the observation sequence $\langle x_{n-k}, \dots, x_{n-1} \rangle$ does not match the condition $\langle s_{1}, \dots, s_{k} \rangle$.
This is because such observations provide no information about $\hat{p}_{(s_{k+1} \given s_{1}, \dots, s_{k})}$.
However, in certain cases, it might be useful to interpret the consistent lack of observations for a particular outcome as a signal to \emph{remove} the existing information (i.e., increase entropy) related to the corresponding conditional probabilities.
We refer to this operation as \emph{entropy regulation}, and it is realized by gradually changing the \acp{CPD} towards uniform distribution $\vec{U}$, as shown in \cref{eq:entropy_regulation},
\begin{equation}
    \label{eq:entropy_regulation}
    \vec{Q}[n] = (1-\beta)\vec{Q}[n-1] + \beta \vec{U}
\end{equation}
where $\vec{Q}$ is a \ac{CPD} and $0<\beta<1$ is the regularization rate.

Entropy regulation is particularly useful when a previously observed symbol $s$ is either never or rarely seen in a new mode.
This could happen due to the transitions leading to $s$ having small probability values, or $s$ not being an element of the set of states defined for the new mode.
In such cases, \cref{eq:ehmc_update_rule} may insufficiently update the transition probabilities conditioned on $s$, leaving residual values that no longer reflect the active mode.
Incorporating \cref{eq:entropy_regulation} alongside \cref{eq:ehmc_update_rule} addresses this issue, albeit with the potential trade-off of additional computational complexity, as discussed in \cref{subsec:mth:complexity}.

The pseudocode for online estimation of conditional probabilities is given in \cref{alg:emc_estimate_p}.
\begin{algorithm}
    \caption{Online Estimation of Transition Probabilities}
    \label{alg:emc_estimate_p}
    \begin{algorithmic}[1]
        \Require $k$, $\vert\Sigma\vert$, $\lambda$, $\beta$
        \State initialize queue $Q$ of observations, uniform tensors $\hat{P}$ and $U$, empty condition vector $c$
        \While{$Q$ is not empty}
            \State get observation $s$ from $Q$
            \If{$|c|=k$} \Comment{wait until $k$ observations}
                \State $\hat{P}(\forall \sigma\in\Sigma \given c) \gets \lambda\hat{P}(\forall \sigma\in\Sigma \given c)$ \Comment{penalty}
                \State $\hat{P}(s \given c) \gets \hat{P}(s \given c)+(1-\lambda)$  \Comment{reward}
                \If{$\beta>0$} \Comment{entropy regulation}
                    \State $U(\forall \sigma\in\Sigma \given c) \gets \hat{P}(\forall \sigma\in\Sigma \given c)$
                    \State $\hat{P} \gets (1-\beta)\hat{P} + \beta U$
                    \State reset $U$
                \EndIf
                \State remove the oldest element from $c$
            \EndIf
            \State append $s$ to $c$
        \EndWhile
    \end{algorithmic}
\end{algorithm}

Next, we discuss how the estimated stochastic tensor can be tracked to detect drifts occurring in data streams. 

\subsection{Drift Detection}

Online estimation of conditional probabilities allows for tracking and detecting drifts in streaming data which later will be used to detect mode changes.
To quantify drift, we use the Hellinger distance \cite{hellinger_NeueBegrundung_1909} as defined in \cref{eq:hellinger_distance}.
\begin{equation}
    \label{eq:hellinger_distance}
    H(P,Q) = \frac{1}{\sqrt{2}} \sqrt{\sum_{i=1}^{k}(\sqrt{p_{i}}-\sqrt{q_{i}})^{2}}
\end{equation}

The Hellinger distance, bounded in $[0,1]$, measures the distance between two probability distributions, with $0$ indicating identical distributions and $1$ indicating disjoint ones.
It can be used to detect changes by periodically comparing the latest estimate $\hat{P}[n]$ with a past version $\hat{P}[n-\tau]$.
Using a $\tau$-step interval for comparisons allows changes in $\hat{P}$ to accumulate and reduces computational costs.
If the distance exceeds a threshold $\delta$, as given in \cref{eq:distance_threshold}, a drift is detected.
\begin{equation}
    \label{eq:distance_threshold}
    H(\hat{P}[n],\hat{P}[n-\tau]) > \delta
\end{equation}

\ac{propAlg} utilizes the result of \cref{eq:distance_threshold} to keep an indicator $\phi$ of steady ($\phi=1$) or drift ($\phi=0$) conditions.
This indicator $\phi$ is utilized in determining when to perform a memory check.
Namely, a memory check is done if $\phi=0$ and $\Delta \hat{P} <= \delta$ (i.e., drift \textrightarrow{} steady) to identify the transitioned mode.
Additionally, to pick up early matches, checks are done during drift conditions ($\phi=0$).
The pseudocode is given in \cref{alg:emc_check_drift}.
\begin{algorithm}
    \caption{Check Drift}
    \label{alg:emc_check_drift}
    \begin{algorithmic}[1]
        \State $\Delta\hat{P} \gets H(\hat{P}, \hat{P}_{\text{prv}})$ \Comment{\cref{eq:hellinger_distance}}
        \If{$\phi=1$}
            \If{$\Delta\hat{P} > \delta$}
                \State $\phi\gets0$ \Comment{drift detected}
                \State \texttt{identify\_regime()} \Comment{\cref{alg:emc_identify_regime}}
            \EndIf
        \Else
            \If{$\Delta\hat{P} < \delta$}
                \State $\phi\gets1$ \Comment{drift ended, now steady}
            \EndIf
            \State \texttt{identify\_regime()} \Comment{\cref{alg:emc_identify_regime}}
        \EndIf
    \end{algorithmic}
\end{algorithm}

The internal variable $\phi$ is also utilized 
\begin{inparaenum}
    \item for keeping the memory of modes read-only during drifts and
    \item in fast/slow learning, which are described in \cref{subsec:mode_discovery_and_recognition,subsec:fast_and_slow_learning}.
\end{inparaenum}

\subsection{Mode Discovery and Recognition}
\label{subsec:mode_discovery_and_recognition}

\ac{propAlg} discovers and recognizes modes of a non-stationary system by maintaining a memory of transition probabilities linked to previously observed modes.
A memory check compares the most recent estimate for transition probabilities with existing models in the memory using the Hellinger distance.
If the distance to the closest mode is below the model similarity threshold $\eta$, \ac{propAlg} updates its prediction to this closest mode. If no mode is sufficiently similar, \ac{propAlg} waits for steady conditions ($\phi=1$) before initializing a new mode in memory, preventing unnecessary entries from transient changes.

Moreover, \ac{propAlg} also performs updates on the modes existing in the memory.
The idea is to utilize the convergence property shown in \cref{thm:arbitrary_order_convergence} by using a incremental mean estimator and update the model in memory with the most recent $\hat{\vec{P}}$ only when $\phi=1$.
This way, with more data becoming available regarding an underlying mode, the representation of that mode in memory is gradually updated to get a more accurate model.
The pseudocode is given in \cref{alg:emc_identify_regime}.
\begin{algorithm}
    \caption{Identify Regime}
    \label{alg:emc_identify_regime}
    \begin{algorithmic}[1]
        \If{$|\mathcal{M}|=0$} \Comment{if memory is empty}
            \If{$\phi=1$}
                \State append $M_{\hat{P}}$ to $\mathcal{M}$ \Comment{save new mode}
            \EndIf
        \Else
            \State $d \gets \min_{M \in \mathcal{M}} H(\hat{P}, M)$ \Comment{dist. to the closest model}
            \If{$d<\eta$}
                \State $M_{\text{pred}} \gets \argmin_{M \in \mathcal{M}} H(\hat{P}, M)$
                \If{$\phi=1$}
                    \State update $M$ using $\hat{P}$ \Comment{incremental mean update}
                \EndIf
            \Else
                \If{$\phi=1$}
                    \State append $M_{\hat{P}}$ to $\mathcal{M}$ \Comment{new mode saved}
                    \State $M_{\text{pred}} \gets M_{\hat{P}}$
                \EndIf
            \EndIf
        \EndIf
    \end{algorithmic}
\end{algorithm}

\subsection{Fast and Slow Learning}
\label{subsec:fast_and_slow_learning}

Utilizing the internal variable $\phi$ about drift and steady conditions, \ac{propAlg} can adaptively switch between fast (exploratory) and slow (exploitative) learning modes, which has been proposed to improve multiplicative learning schemes \cite{coskun_AdaptiveEstimation_2022}.
These two learning modes have different learning coefficients $\lambda_{f}$ and $\lambda_{s}$.
This mechanism is useful since
\begin{inparaenum}
    \item with $\lambda_{f}$, conditional probabilities approach true values faster after regime changes, and
    \item with $\lambda_{s}$, the variance of estimates is lower during each regime.
\end{inparaenum}
In addition to $\lambda$, the previously introduced deviation threshold $\delta$ and model similarity threshold $\eta$ may have different values for fast and slow learning modes.

\subsection{Complexity Analysis}
\label{subsec:mth:complexity}

A transition probability tensor of a $k$th order Markov chain with $m$ states includes $m^{k+1}$ transition probabilities.
Following each observation, only the relevant \ac{CPD}, which has $m$ elements, is updated by \cref{eq:ehmc_update_rule}.
However, the computational complexity of this update is $O(mk)$ due to the access time required for a $k$-dimensional tensor. 
If entropy regularization is used, the full tensor is updated, which brings the computational complexity to $O(m^{k+1})$.
Computation of the Hellinger distance between two transition probability tensors is $O(m^{k+1})$, which is done in every $\tau$ time steps, and additionally $|\mathcal{M}|$ times when a drift is detected.
Runtime tests are given in Supplemental Fig. 1 and 2.

\subsection{How to Set the Parameters?}

Up to now we have introduced learning coefficient $\lambda$, entropy regularization parameter $\beta$, drift threshold $\delta$, model distance threshold $\eta$ and drift checking frequency $\tau$.

Regarding $\lambda$, it is stated in \cite{oommen_StochasticLearningbased_2006} that values in $[0.9, 0.99]$ are suggested for multi-action \ac{LA} schemes.
Thus, setting the fast learning coefficient $\lambda_{f}$ close to $0.9$ and the slow learning coefficient $\lambda_{s}$ close to $0.99$ is appropriate.
Regarding the regularization parameter $\beta$, due to the additional computational complexity, we recommend setting it to $0$ if all states $s\in\Sigma$ are expected to be visited in all modes.

Parameters $\delta$ and $\eta$ are thresholds based on the same distance measure: $\delta$ triggers transitions between steady and drift conditions, while $\eta$ assesses proximity to stored modes.
A lower $\delta$ increases drift detection, but mode prediction relies on distances to stored modes, adding robustness against incorrect switches.
Based on our findings, we recommend $\delta\in[0.1,0.3]$ and $\eta\in[0.25,0.5]$, with slightly lower values when modes show similar transitions.
While not necessary, $\delta$ and $\eta$ can be set to slightly higher values in fast learning case due to the larger magnitude of updates on $\hat{P}$.

The parameter $\tau$ determines the frequency of drift checks and can be linked to $\lambda$ so that $\tau\approx1/(1-\lambda)$, which corresponds to $\tau=10$ for $\lambda=0.9$ and $\tau=100$ for $\lambda=0.99$.

As previously discussed, \ac{propAlg} operates in an unsupervised manner. However, if labels are available for at least two modes, parameter values can be selected using Bayesian optimization \cite{optuna_2019} so that \ac{propAlg} distinguishes these modes.
This approach allows for the identification of an effective parameter set tailored to a specific task, even with labels from a single mode transition.
In practice, this method has shown to be effective in real-world applications of \ac{propAlg}, particularly in the condition monitoring of electric motors (refer to \cref{subsec:cwru}).

\section{Experimental Evaluation}
\label{sec:experimental_evaluation}

In this section, we illustrate the proposed method on synthetic and real-world data.
The hallmark of \ac{propAlg} is that it can utilize the temporal structure that exists in the sequential data for detecting the underlying modes of a non-stationary system.
Therefore, we dedicate separate sets of experiments on synthetic data to address the following key questions:
\begin{itemize}
    \item How well the temporal structure information is extracted given a non-stationary sequence (i.e., conditional probability tracking)?
    \item How accurately the changes in the temporal structure are detected (i.e., change point detection)?
    \item How well the underlying modes are detected (i.e, mode discovery and recognition)?
\end{itemize}

In addition to the synthetic experiments, \ac{propAlg} is evaluated on real-world applications including
\begin{inparaenum}
    \item human activity recognition from accelerometer data,
    \item condition monitoring on electric motors from vibration data, and
    \item human eye-state detection from \ac{EEG} data.
\end{inparaenum}

Datasets used in real-world applications are all publicly available.
Additionally, implementation of the proposed method and the synthetic data generators are also available\footnote{\texttt{\url{https://bckrlab.org/p/emc}}}.

\subsection{Experiments on Synthetic Data}

We generate data by sampling sequences from random modes, defined by the following parameters:
\begin{inparaenum}[(i)]
    \item the number of modes $|\mathcal{M}|$,
    \item the Markov order $k$,
    \item the number of states $|\Sigma|$,
    \item regime duration distribution $P_{d}$, and 
    \item number of regimes $R$.
\end{inparaenum}
Transition probabilities are randomly generated using a flat Dirichlet distribution.
We use the same values of $k$, $\Sigma$ and $P_{d}$ across all $M \in \mathcal{M}$, making the problem more challenging since no state is exclusive to a subset of modes.

To generate synthetic data streams, first we generate $|\mathcal{M}|$ $k$th order Markov chains, each with $|\Sigma|$ states.
Then, we randomly generate an $R$-long sequence of mode indices (i.e., $R$ regimes), and for each regime we sample a sequence of symbols from the corresponding Markov chain.
Regime lengths are sampled from $P_{d}$.
As a result, we obtain a symbol sequence with an expected total length of $\mathbb{E}(P_{d})\times R$.

In accordance with the key questions discussed earlier, we group the synthetic experiments into three categories, which are
\begin{inparaenum}[(i)]
    \item probability tracking,
    \item change point detection, and
    \item mode discovery and recognition.
\end{inparaenum}

\subsubsection{Probability Tracking}

% SETUP
In this set of experiments, we aim to demonstrate the temporal structure tracking capabilities of \ac{propAlg} by comparing estimated and true conditional probabilities across different modes of a non-stationary system.

% DATA
Synthetic data streams are generated with the following parameters: the number of modes $|\mathcal{M}|$ is $5$ and each of these modes is a first-order Markov chain ($k=1$) with $4$ states ($|\Sigma|=4$).
From these modes, $R=10$ regimes are generated.
To measure the effect of regime durations, three different regime duration distributions (short/medium/long) are used: $P_{d}=\{\mathcal{U}(500,1000),
\mathcal{U}(1500,2000),\mathcal{U}(2500,3000)\}$, where $\mathcal{U}(a,b)$ is a uniform distribution between $a$ and $b$.

% METHODS
We compare the proposed method with
\begin{inparaenum}[(i)]
    \item \ac{MC-SW} \cite{ren_AnomalyDetection_2017,tilahun_CooperativeMultiagent_2017}, which constructs the chain from the data inside a fixed-size sliding window, and
    \item \ac{MC-ADWIN}, which utilizes \ac{ADWIN} algorithm \cite{bifet_LearningTimeChanging_2007} to make window size adaptive.
\end{inparaenum}

% PARAMS
The test set includes \num{100} sequences generated with different seeds.
An additional \num{10} sequences are used to optimize each method's hyperparameters via Bayesian optimization, with results provided in Supplementary Table 1.
For \ac{MC-SW}, we add cases with shorter ($w=100$) and longer ($w=500$) windows to illustrate effects on various regime durations.
All methods are set to build a first-order Markov chain ($k=1$).

% EVALUATION
Performance evaluation is done by comparing the estimated conditional probabilities with true ones at each time point in terms of \ac{AE}.
Mean and cumulative \ac{AE} are used to report single values representing the performance of a method on a sequence.
% RESULTS
The results are given in \cref{tab:mss_tracking}, and an example run is visualized in \cref{fig:tracking_cae}.

\begin{table}[ht]
    \centering
    \footnotesize
    \renewcommand{\arraystretch}{1.2}
    \pgfplotstableread[
        col sep=comma,
        header=false,
        skip first n=1
        ]{syn_pt_mae.csv}\data
    \pgfplotstabletypeset[
        columns={0,1,2,3},
        columns/0/.style={
            string type,
            column type={l},
            column name={},
        },
        columns/1/.style={
            string type,
            column type=c,
            column name={$\mathcal{U}(500,1000)$},
        },
        columns/2/.style={
            string type,
            column type=c,
            column name={$\mathcal{U}(1500,2000)$},
        },
        columns/3/.style={
            string type,
            column type={c},
            column name={$\mathcal{U}(2500,3000)$},
        },
        bm/.style={postproc cell content/.style={@cell content=\bm{##1}}},
        bf/.style={postproc cell content/.style={@cell content=\textbf{##1}}},
        ul/.style={postproc cell content/.style={@cell content=\underline{##1}}},
        every row 0 column 1/.style={bm},
        every row 0 column 2/.style={bm},
        every row 0 column 3/.style={bm},
        every row 3 column 1/.style={ul},
        every row 3 column 2/.style={ul},
        every row 4 column 3/.style={ul},
        every head row/.style={%
            before row={%
                \toprule
                \multirow{2.4}{*}{\textbf{Method}} & \multicolumn{3}{c}{\textbf{Regime Duration}} \\
                \cmidrule(lr){2-4}
            },
            after row={\midrule}
        },
        every last row/.style={after row={\bottomrule}}
    ]{\data}
    \caption{
        \Acl{MAE} between true and estimated conditional probabilities.
        Each cell contains the mean $\pm$ standard deviation of $100$ runs.
        Best values are shown in bold font.
        Second-best values are underlined.
    }
    \label{tab:mss_tracking}
\end{table}

\begin{figure}[htbp]
    \centering
    \includegraphics{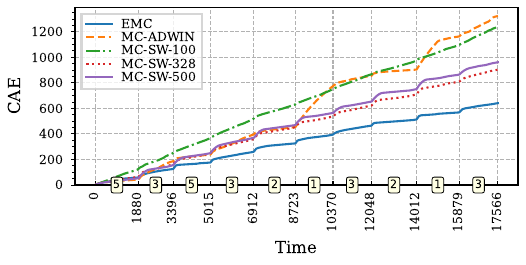}
    \caption{
        \Acf{CAE} between true and estimated conditional probabilities for $P_{d}=\mathcal{U}(1500,2000)$.
        The numbers inside the boxes along the horizontal axis indicate the mode index $i$ of $M_{i} \in \mathcal{M}$.
    }
    \label{fig:tracking_cae}
\end{figure}

% DISCUSSION
\cref{tab:mss_tracking} presents results based on different regime duration distributions.
Considering the fixed-size sliding-window method \ac{MC-SW}, as expected, shorter windows are more useful on shorter regimes due to quicker adaptation to rapid changes.
As the regime duration increases, we observe that $w=500$ outperforms $w=328$ due to the increased utilization of data, which outweighs the benefits of faster adaptation.

The adaptive-window-size approach (\ac{MC-ADWIN}) underperformed compared to fixed-window-size methods.
Further inspection revealed that \ac{MC-ADWIN} continually increases its window size due to undetected changes.
This issue arises from the change detection mechanism of \ac{ADWIN}, which is based on identifying changes in data distribution rather than the temporal structure of the data.
For example, two consecutive modes with entirely different temporal structures could share the same data distribution, illustrating the need to leverage temporal information for adapting switching modes.

\ac{propAlg} obtained the smallest error values for \emph{all} types of regime duration distributions with the \emph{same} parameter values.
This is also visible in the result of an example run given in \cref{fig:tracking_cae}, where the cumulative error obtained by the proposed method is the smallest compared to other approaches.

\subsubsection{Change Point Detection}

% SETUP
This set of experiments focus on measuring how accurately changes in the temporal structure are detected on data streams.
We employ the \emph{classification view} \cite{burg_EvaluationChange_2020} to change point detection task, that is, each time point is either a change point or not, and the goal is to correctly classify them into these two categories.

% DATA
Synthetic data streams are generated with the same parameters as those employed in probability tracking experiments.
In this case, only medium length regimes, $P_{d}=\mathcal{U}(1500,2000)$, are used.
Therefore, the expected total length of a sequence is $15000$ and it involves $9$ change points.

% METHODS
We compare the proposed method with online change point detection algorithms from the literature, namely 
\begin{inparaenum}[(i)]
    \item \ac{ADWIN} \cite{bifet_LearningTimeChanging_2007}
    \item \ac{KSWIN} \cite{raab_ReactiveSoft_2020}
    \item \ac{PHT} \cite{page_ContinuousInspection_1954,gama_EvaluatingStream_2013}, and
    \item \ac{SCD} \cite{aslanci_DetectionRegime_2017}.
\end{inparaenum}

% PARAMS
Similar to the probability tracking experiments, the test set involves \num{100} sequences generated with different seeds, and a separate set of \num{10} sequences are used for hyperparameter optimization.
The values are given in Supplementary Table 2.

% EVALUATION
We utilize two key metrics for evaluation.
First, the $F_{1}$ score assesses the accuracy of detections, considering a detection as true positive if it occurs within a \ac{MOE}, after the actual change point.
In this case, \ac{MOE} is set to $1/6$th of the shortest regime, which is $250$.
Note that we label detections occurring prior to the true change point as false positive, since there is no indication of a regime change in the data before the actual point.
Second, we utilize the \emph{mean detection lag}, representing the average difference between true and detected points for all true positives.
% RESULTS
The results are given in \cref{tab:syn_cpd}.

\begin{table}[ht]
    \centering
    \footnotesize
    \renewcommand{\arraystretch}{1.2}
    \pgfplotstableread[
        col sep=comma,
        header=false,
        skip first n=1
        ]{syn_cpd_f1.csv}\data
    \pgfplotstabletypeset[
        columns={0,1,2,3,4},
        columns/0/.style={
            string type,
            column type={l},
            column name={\textbf{Method}},
        },
        columns/1/.style={
            string type,
            column type=c,
            column name={\textbf{\makecell{F1\\Score (\textuparrow)}}},
        },
        columns/2/.style={
            string type,
            column type=c,
            column name={\textbf{\makecell{False\\Negatives (\textdownarrow)}}},
        },
        columns/3/.style={
            string type,
            column type=c,
            column name={\textbf{\makecell{False\\Positives (\textdownarrow)}}},
        },
        columns/4/.style={
            string type,
            column type={c},
            column name={\textbf{\makecell{Mean Detection\\Lag (\textdownarrow)}}},
        },
        bm/.style={postproc cell content/.style={@cell content=\bm{##1}}},
        bf/.style={postproc cell content/.style={@cell content=\textbf{##1}}},
        ul/.style={postproc cell content/.style={@cell content=\underline{##1}}},
        every row 0 column 0/.style={postproc cell content/.style={@cell content={##1 \cite{bifet_LearningTimeChanging_2007}}}},
        every row 1 column 0/.style={postproc cell content/.style={@cell content={##1 (Ours)}}},
        every row 2 column 0/.style={postproc cell content/.style={@cell content={##1 \cite{raab_ReactiveSoft_2020}}}},
        every row 3 column 0/.style={postproc cell content/.style={@cell content={##1 \cite{page_ContinuousInspection_1954,gama_EvaluatingStream_2013}}}},
        every row 4 column 0/.style={postproc cell content/.style={@cell content={##1 \cite{aslanci_DetectionRegime_2017}}}},
        every row 1 column 1/.style={bm},
        every row 1 column 2/.style={bm},
        every row 1 column 3/.style={bm},
        every row 4 column 4/.style={bm},
        every row 4 column 1/.style={ul},
        every row 4 column 2/.style={ul},
        every row 4 column 3/.style={ul},
        every row 1 column 4/.style={ul},
        % after row = {[2pt]},
        every head row/.style={before row={\toprule}, after row={\midrule}},
        every last row/.style={after row={\bottomrule}}
    ]{\data}
    \caption{
        Change point detection results.
        Each cell contains the mean $\pm$ standard deviation of $100$ runs.
        Best values are shown in bold font.
        Second-best values are underlined.
    }
    \label{tab:syn_cpd}
\end{table}

% DISCUSSON
As shown in \cref{tab:syn_cpd}, \ac{propAlg} obtained the highest $F_{1}$ score along with the fewest false negatives and false positives, and the lowest standard deviation, highlighting its detection consistency.
Note that mean detection lag is measured \emph{only} for true positives, so it is best assessed in conjunction with the $F_{1}$ score.
For example, a low detection lag with poor overall detection accuracy does not necessarily indicate strong performance -- it only suggests that when a change is detected, it's detected relatively quickly.
Given that, \ac{propAlg} stands out as the most reliable overall, as it pairs a near-best mean detection lag with the highest $F_{1}$ score. 

The low performance of \ac{ADWIN} mirrors that of \ac{MC-ADWIN}, as mode changes are not always apparent when only data distribution is considered.
\ac{SCD}'s second-best $F_{1}$ score is expected, given its use of first-order Markov probabilities for change detection.
\ac{propAlg} enhances \ac{SCD}’s mechanism by incorporating fast/slow learning and the additional robustness due to having a memory for drift/change differentiation.

\subsubsection{Mode Discovery and Recognition}

% SETUP
In this case, the task is to correctly recognize different modes in the data stream, which encompasses the goals of the first two experiment types, namely the method is expected to
\begin{inparaenum}
    \item build models from data,
    \item detect change points, and
    \item discover and recognize regimes that originate from the same mode.
\end{inparaenum}

% DATA
Data streams are generated with the same parameters as those employed in change point detection experiments.
% METHODS
We compare \ac{propAlg} with \ac{EPSTM} \cite{surmeli_MultivariateTime_2020}, which works \emph{offline} and requires change points given to it.
Since an online method inevitably requires some time to recognize different modes, to set up a fair comparison, we provide \emph{delayed} change points to \ac{EPSTM}.
The delay is set as the \emph{mean recognition lag} of \ac{propAlg} representing the average time between true change points and the recognition of the new regime.
Additionally, to demonstrate the scenario with perfect change points, we run another instance of \ac{EPSTM} using the true change points.
We also run CluStream \cite{aggarwal_FrameworkClustering_2003} and DBStream \cite{hahsler_ClusteringData_2016} on one-hot-encoded data, representing partitioning and density based clustering approaches.

% PARAMS
Similar to previous experiments, the test set involves \num{100} sequences generated with different seeds and a separate set of \num{10} sequences are used to select the hyperparameters of methods using Bayesian optimization.
The parameter values used are given in Supplementary Table 3.

% EVALUATION
Due to the similarity of this task to clustering, to measure the performance of methods we use \ac{ARI}, a metric that assesses the similarity between the true and predicted clustering, corrected for chance.
% RESULTS
Results are given in \cref{tab:mss_mdr}.
Also, the transition plots of true and discovered modes for an example run is depicted in \cref{fig:mdr_transition}.

\begin{table}[ht]
    \centering
    \footnotesize
    \renewcommand{\arraystretch}{1.2}
    \pgfplotstableread[
        col sep=comma,
        header=false,
        skip first n=1
        ]{syn_mdr_ari.csv}\data
    \pgfplotstabletypeset[
        columns={0,1,2,3},
        columns/0/.style={
            string type,
            column type={l},
            column name={\textbf{Method}},
        },
        columns/1/.style={
            string type,
            column type={c},
            column name={\textbf{\makecell{Operation\\Type}}},
        },
        columns/2/.style={
            string type,
            column type={c},
            column name={\textbf{\makecell{Change\\Points}}},
        },
        columns/3/.style={
            string type,
            column type={c},
            column name={\textbf{ARI (\textuparrow)}},
        },
        bm/.style = {postproc cell content/.style={@cell content/.add={\boldmath}{}}},
        every row 0 column 0/.style={postproc cell content/.style={@cell content={##1 (Ours)}}},
        every row 1 column 0/.style={postproc cell content/.style={@cell content={##1 \cite{aggarwal_FrameworkClustering_2003}}}},
        every row 2 column 0/.style={postproc cell content/.style={@cell content={##1 \cite{hahsler_ClusteringData_2016}}}},
        every nth row={3}{before row=\midrule},
        every row 3 column 0/.style={postproc cell content/.style={@cell content={##1 \cite{surmeli_MultivariateTime_2020}}}},
        every row 4 column 0/.style={postproc cell content/.style={@cell content={##1 \cite{surmeli_MultivariateTime_2020}}}},
        every row 0 column 3/.style={bm},
        every row 3 column 3/.style={bm},
        every head row/.style={before row={\toprule}, after row={\midrule}},
        every last row/.style={after row={\bottomrule}}
    ]{\data}
    \caption{
        Mode discovery and recognition results.
        Each cell contains the mean $\pm$ standard deviation of $100$ runs.
        Best values for online and offline methods are shown in bold font. 
    }
    \label{tab:mss_mdr}
\end{table}

\begin{figure}[htbp]
    \centering
    \includegraphics{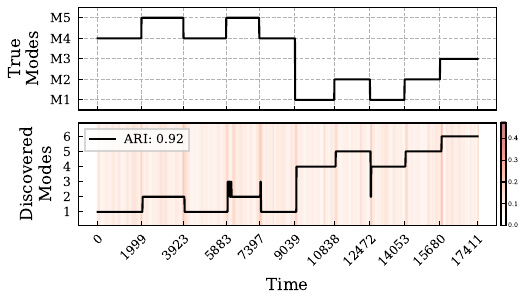}
    \caption{
        Transition plots of true modes (top) and detected modes (bottom) on synthetic data.
        The heatmap in the bottom plot represents the Hellinger distance, quantifying the detected drift over time.
    }
    \label{fig:mdr_transition}
\end{figure}

% DISCUSSION
As seen in \cref{tab:mss_mdr}, when true change points are available, \ac{EPSTM} gets $0.98$ \ac{ARI}.
Although \ac{propAlg} inevitably suffers from recognition lag due to being an online method, it managed to achieve $0.85$ ARI.
When the same amount of lag is introduced to change points given to \ac{EPSTM}, its performance significantly decreases.
The stream clustering methods CluStream and DBStream managed to get ARI values that are slightly better than random assignment, showing the need for utilizing the temporal structure of the data in \ac{OMDR} task.

\Cref{fig:mdr_transition} shows how \ac{propAlg} discovers modes and distinguishes regimes originating from different sources.
Note that since \ac{propAlg} builds a memory as data arrive, the indices of discovered modes not necessarily match those of the system modes (e.g., discovered mode $1$ does not correspond to system mode \texttt{M1}, rather corresponds to \texttt{M4}).
Although there are incorrect switches during the system mode transitions from \texttt{M4} to \texttt{M5}, \texttt{M5} to \texttt{M4} and \texttt{M2} to \texttt{M1}, \ac{propAlg} successfully identified the correct system mode once more data became available.

\subsection{Experiments on Human Activity Data}

We evaluate \ac{propAlg} on \texttt{RealWorld (HAR)} dataset \cite{sztyler_OnbodyLocalization_2016} which covers \num{3}-dimensional acceleration data collected from \num{15} subjects while performing the following activities:
\begin{inparaenum}[(i)]
    \item climbing stairs down,
    \item climbing stairs up,
    \item jumping,
    \item lying,
    \item standing,
    \item sitting,
    \item running/jogging, and
    \item walking.
\end{inparaenum}
The goal is to test whether \ac{propAlg} is able to discover individual activities online, without prior knowledge of change points or the number of activities.

We follow the \ac{SPR} methodology as outlined in \cite{fu_SyntacticPattern_1977} to preprocess the \num{3}-dimensional continuous sensor data, transforming it into a discrete sequence.
In the lexical analysis part of a typical \ac{SPR} application, short temporal patterns, referred to as primitives, are extracted from the original time-series data.
These primitives are basically the building blocks of the data, and the collection of these patterns is referred to as the \emph{alphabet}.
The original signal is then reconstructed using the primitives, resulting in a sequence of primitive indices, which serves as input to \ac{propAlg}.

To discretize the \num{3}-dimensional continuous data, we employ the algorithm proposed in \cite{coskun_SyntacticPattern_2022} with a slight modification.
Specifically, the extracted primitives are \emph{composite}, that is, they consist of equal-length segments from each feature of the sensor data---here, the $x$, $y$ and $z$ components of the accelerometer data.
For this dataset, we build an alphabet of composite primitives based on the data from a single hold-out subject (Subject \num{1}).
The primitives have a length of \num{15} and are obtained by clustering time-series segments using Breathing $k$-means \cite{fritzke_BreathingKMeans_2021} into \num{7} clusters, which is determined by the Davies-Bouldin score \cite{davies_ClusterSeparation_1979}.
Then, the original data is reconstructed by iteratively selecting the best-fitting cluster, producing a sequence of cluster indices.
The parameters are optimized on the same hold-out data using Bayesian optimization, and the values are given in Supplementary Table 4.

For the remaining subjects (Subjects $2$--$15$), randomly selected sequences of $10$ activities are used, and performance on these sequences is evaluated using the \ac{ARI} metric.
Additionally, we calculate a variant of \ac{ARI} that only considers predictions made in the steady condition (i.e., $\phi=1$), to assess the improvement in performance if only confident predictions are included.
We also report the proportion of predictions made in the drift condition (i.e., $\phi=0$).
The results are presented in \cref{tab:rwhar} and an experiment run is visualized in \cref{fig:har_mt}.

\begin{table}[ht]
    \centering
    \footnotesize
    \renewcommand{\arraystretch}{1}
    \pgfplotstableread[
        col sep=comma,
        header=false,
        skip first n=1
        ]{har_ari.csv}\data
    \pgfplotstabletypeset[
        columns={0,1,2,3},
        columns/0/.style={
            string type,
            column type={c},
            column name={\textbf{\makecell{Subject\\ID}}}
        },
        columns/1/.style={
            string type,
            column type={c},
            column name={\textbf{ARI}},
        },
        columns/2/.style={
            string type,
            column type={c},
            column name={\textbf{\makecell{ARI\\(for only $\phi=1$)}}},
        },
        columns/3/.style={
            string type,
            column type={c},
            column name={\textbf{\makecell{Drift\\Ratio}}},
        },
        bm/.style={postproc cell content/.style={@cell content=\bm{##1}}},
        bf/.style={postproc cell content/.style={@cell content=\textbf{##1}}},
        every row 14 column 0/.style={bf},
        every row 14 column 1/.style={bm},
        every row 14 column 2/.style={bm},
        every row 14 column 3/.style={bm},
        every head row/.style={before row={\toprule}, after row={\midrule}},
        every last row/.style={before row={\midrule}, after row={\bottomrule}}
    ]{\data}
    \caption{
        Mode discovery and recognition results on human activity data.
        Each cell is the mean $\pm$ standard deviation of $10$ runs.
    }
    \label{tab:rwhar}
\end{table}

\begin{figure}[htbp]
    \centering
    \includegraphics{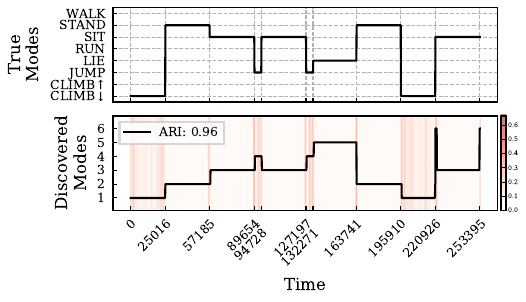}
    \caption{
        Transition plots of true modes (top) and detected modes (bottom) on activity data of subject \num{9}.
        The heatmap in the bottom plot represents the Hellinger distance, quantifying the detected drift over time.
    }
    \label{fig:har_mt}
\end{figure}

% DISCUSSION
The results in \cref{tab:rwhar} demonstrate that, despite using an alphabet constructed from the data of a single subject, \ac{propAlg} achieves an average \ac{ARI} of $0.72$.
Furthermore, as shown in \cref{fig:har_mt}, the estimated probabilities for certain actions (in this case CLIMB↓ and JUMP) exhibit significant fluctuations, even in the absence of a regime change.
Nevertheless, despite high deviation, \ac{propAlg} consistently remained in the correct predicted mode, highlighting the robustness of the proposed method.

\subsection{Experiments on Electric Motor Vibration Data}
\label{subsec:cwru}

We also test the proposed method on \ac{CWRU} Bearing Dataset \cite{CWRU} which involves vibration data collected from electric motors with healthy and faulty bearings under various load conditions.
The severity of the fault is determined by the diameter of the damage on the bearing, which impacts the vibration behavior of the motor.
The possible diameter values are \num{0.007}, \num{0.014}, \num{0.021} and \num{0.028} inches, with load conditions ranging from \num{0} to \num{3} horsepower.
The data is sampled at a rate of \SI{12}{\kilo\hertz}.

The experimental setup mirrors that of the human activity experiments: \ac{propAlg} is tested on sequences of different fault modes to evaluate its ability to detect and differentiate between them.
In particular, we focus on a scenario where the fault progressively worsens in terms of the damage diameter.

The construction of the primitive alphabet and the hyperparameter optimization are done on a portion of the data that only includes a single transition from one fault condition to another without load, namely \texttt{FD:07} \textrightarrow{} \texttt{FD:28}, which means the fault diameter increases from \num{0.007} to \num{0.028} inches.
Since the data is univariate, we directly apply the alphabet construction algorithm from \cite{coskun_SyntacticPattern_2022}.
In this case, the primitives are of length \num{2}, and are obtained by clustering the time-series segments using Breathing $k$-means \cite{fritzke_BreathingKMeans_2021} into \num{27} clusters, which is again determined by the Davies-Bouldin score \cite{davies_ClusterSeparation_1979}.

For this experiment, a \emph{second-order} \ac{propAlg} is used with the hyperparameters optimized via Bayesian optimization on the hold-out data and given in Supplementary Table 4.

The test set consists of data with transitions from 
\begin{inparaenum}[(i)]
    \item healthy to faulty modes, and
    \item healthy to faulty and then another (more severe) faulty modes.
\end{inparaenum}
The results are presented in \cref{tab:cwru} with a mode transition plot shown in \cref{fig:cwru_mt}.

\begin{table}[htbp]
    \centering
    \footnotesize
    \renewcommand{\arraystretch}{1}
    \pgfplotstableread[
        col sep=comma,
        header=false,
        skip first n=1
        ]{cwru_ari.csv}\data
    \pgfplotstabletypeset[
        columns={1,0,2,3,4},
        columns/0/.style={
            string type,
            column type={r},
            column name={\textbf{\makecell{Mode Transition\\Scenario}}}
        },
        columns/1/.style={
            string type,
            column type={c},
            column name={\textbf{\makecell{Load}}},
            assign cell content/.code={
                \pgfmathparse{int(Mod(\pgfplotstablerow,7)}%
                \ifnum\pgfmathresult=0%
                    \pgfkeyssetvalue{/pgfplots/table/@cell content}{
                        \multirow{7}{*}{##1}
                    }
                \else
                    \pgfkeyssetvalue{/pgfplots/table/@cell content}{}
                \fi
            },
        },
        columns/2/.style={
            fixed zerofill,
            precision=2,
            column type={c},
            column name={\textbf{ARI}},
        },
        columns/3/.style={
            fixed zerofill,
            precision=2,,
            column type={c},
            column name={\textbf{\makecell{ARI\\(for only $\phi=1$)}}},
        },
        columns/4/.style={
            fixed zerofill,
            precision=2,,
            column type={c},
            column name={\textbf{\makecell{Drift\\Ratio}}},
        },
        bf/.style={postproc cell content/.style={@cell content=\textbf{##1}}},
        bm/.style = {postproc cell content/.style={@cell content/.add={\boldmath}{}}},
        every head row/.style={before row={\toprule}, after row={\midrule}},
        every nth row={7}{before row=\midrule},
        every last row/.style={after row={\bottomrule}},
        every row 28 column 0/.style={bf},
        every row 28 column 2/.style={bm},
        every row 28 column 3/.style={bm},
        every row 28 column 4/.style={bm},
    ]{\data}
    \caption{
        Results on \acs{CWRU} Bearing dataset.
        The test cases cover transitions from the normal mode (i.e., OK) to various fault conditions as well as transitions from one fault to another more severe one.
        Each transition scenario is tested with different loads (\num{0} to \num{4} HP).
    }
    \label{tab:cwru}
\end{table}

\begin{figure}[htbp]
    \centering
    \includegraphics{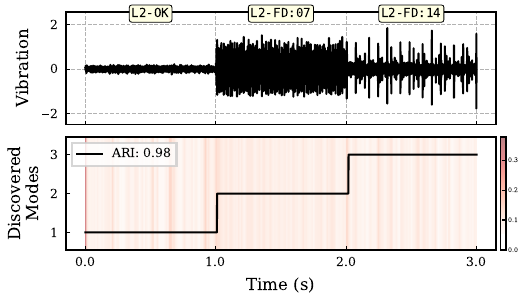}
    \caption{
        Result on vibration data under \num{2} HP of load with mode transition sequence of \texttt{OK} \textrightarrow{} \texttt{FD:07} \textrightarrow{} \texttt{FD:14}.
        The labels of each condition is given along the top horizontal axis.
        The heatmap in the bottom plot represents the Hellinger distance, quantifying the detected drift over time.
    }
    \label{fig:cwru_mt}
\end{figure}

Results given in \cref{tab:cwru,fig:cwru_mt} show that \ac{propAlg} successfully distinguishes different fault conditions.
Given that the primitive alphabet construction and the hyperparameter optimization processes are done only on a single mode switch between two fault modes, from an anomaly detection perspective, the results highlight the online unsupervised anomaly detection capabilities of the proposed method.
Furthermore, testing under various load conditions demonstrates the domain adaptation potential of the proposed method, which reflects a more realistic scenario as different use cases impose varying loads on an electric motor during operation.

\subsection{Experiments on Electroencephalography (EEG) Data}

Finally, we evaluate \ac{propAlg} on \ac{EEG} measurements from the \texttt{EEG Eye State} dataset \cite{misc_eeg_eye_state_264}, which covers \num{14} \ac{EEG} channels with an additional binary signal representing the eye state (i.e., open or closed), which is added manually after inspecting the video recordings.
The goal is to test whether the proposed method manages to find these two modes and correctly switch between them as the eye state of the person actually changes.

A discrete sequence is obtained by the microstate segmentation algorithm presented in \cite{pascual-marqui_SegmentationBrain_1995} and implemented in \cite{pycrostates}, with $|\Sigma|=8$.
The sequence is then streamed to a first-order \ac{propAlg} with hyperparameters (given in Supplementary Table 4) obtained by Bayesian optimization on the first \num{3000} (\SI{20}{\percent}) instances of the sequence.
The \ac{ARI} is calculated for the rest (\SI{80}{\percent}) of the sequence.
The result is given in \cref{fig:eeg_mt}.

\begin{figure}[htbp]
    \centering
    \includegraphics{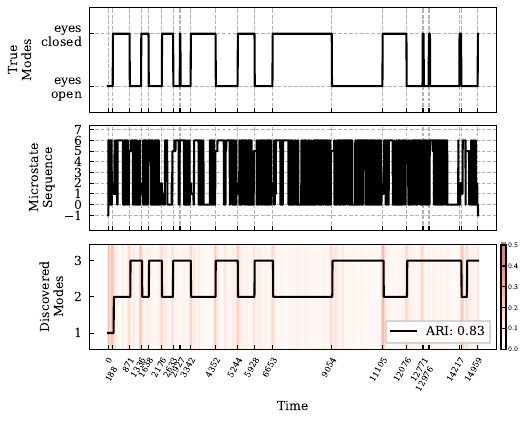}
    \caption{
        Results of \ac{EEG} eye-state task.
        \ac{ARI} becomes \num{0.9} if only the predictions made during steady condition are considered.
        In this case, first \SI{20}{\percent} of the sequence is used for hyperparameter optimization, and the \ac{ARI} is calculated for the rest of the sequence.
        The heatmap in the bottom plot represents the Hellinger distance, quantifying the detected drift over time.
    }
    \label{fig:eeg_mt}
\end{figure}

The discovered modes \texttt{2} and \texttt{3} in \cref{fig:eeg_mt} correspond to closed and open eye states, respectively.
The transitions between these modes largely align with the actual eye state changes, except for a few brief intervals, such as at $t=12771$ and $t=12976$.
A practical application of this result could be the automatic annotation of \ac{EEG} signals as they are being recorded, since human errors, such as  misplaced or forgotten cursors indicating eye states, are common.

\section{Discussion and Future Work}
\label{sec:discussion_and_future_work}

Online tasks such as change point detection and mode discovery/recognition present distinct challenges while offering significant practical value in real-time applications.
Through experimental evaluation of \ac{propAlg}, we show the importance of the temporal structure in input data and demonstrate potential applications utilizing it.
Our synthetic experiments cover key properties of inherent temporal structure in data sequences, and the results demonstrate that \ac{propAlg} effectively leverages these structures to detect changes and discover/recognize the underlying modes of non-stationary systems.
For the \ac{OMDR} task, the baseline stream clustering methods are ineffective because they do not incorporate temporal structure information when determining cluster assignments.
Additionally, results from real-world applications indicate that \ac{propAlg} successfully identifies underlying modes across various tasks in different domains.

In real-world applications, we used a relatively small portion of each dataset for hyperparameter optimization: $6.7\%$ of the human activity data, $10\%$ of the electric motor data and $20\%$ of the \ac{EEG} data).
This approach aims to show that \ac{propAlg} can perform well with minimal data for tuning, making it applicable in scenarios where data for optimization is limited.
Furthermore, in human activity and electric motor condition monitoring experiments, the subsets used for hyperparameter optimization had different characteristics than the test sets.
This highlights the generalization and transferability capability of the method and its robustness in handling domain shifts.

Potential extensions of \ac{propAlg} may include making it more dynamic with regard to unknown symbols.
That is, currently \ac{propAlg} requires a predefined number symbols, which correspond to the states of the learned Markov chain.
It is possible to extend \cref{eq:ehmc_update_rule} so that it combines/splits states and the corresponding probabilities.
Nevertheless, for the scope of this study the fixed-state-space assumption is kept to provide the first results regarding the asymptotic convergence of the update equation and possible real-world applications.
Additionally, \ac{propAlg} currently requires discrete data.
In real-world experiments, we addressed this by clustering the multivariate data to convert it into sequences of discrete values (i.e., cluster indices).
However, extending \ac{propAlg} to directly work with multidimensional continuous data would improve its flexibility and applicability across a wider range of scenarios.
Finally, although \cref{eq:ehmc_update_rule} can estimate arbitrary order conditional probabilities, it inevitably suffers from the exponential increase in number of probabilities as the order increases.
This is partially compensated by updating only a fixed region of the probability tensor, however the space complexity is still exponential.
Finally, making \ac{propAlg} more dynamic with regard to its state space may also improve its applicability in more complex settings, i.e., by adaptively selecting the Markov order from data.

\section{Conclusion}
\label{sec:conclusion}

With this study, we introduced \acfp{propAlg}, which learn from streaming categorical data to discover and recognize underlying operational modes.
We believe that the way that \ac{propAlg} efficiently extracts and utilizes the inherent temporal structure in data streams will pave the way for new methodologies in the \ac{OMDR} task defined in this study, potentially enhancing both theoretical understanding and practical applications.
As demonstrated in experiments, the small amount of data required to optimize the hyperparameters of \ac{propAlg} aligns well with few-shot learning principles, making it easily applicable to real-world tasks with minimal effort.

\section*{Acknowledgments}

The authors would like to thank MInD and Becker Lab members for their valuable feedback.
Funding was provided by the BMBF (01IS22077).

\printbibliography

\clearpage
\appendix
\renewcommand{\thefigure}{A\arabic{figure}}
\setcounter{figure}{0}
\renewcommand{\thetable}{A\arabic{table}}
\setcounter{table}{0}

\section{Supplemental Material}

\subsection{Supplemental Figures}

\begin{figure}[htbp]
    \centering
    \subfloat[Change in runtime with increasing Markov order.\label{fig:eq_kt}]{%
        \includegraphics[width=0.45\linewidth]{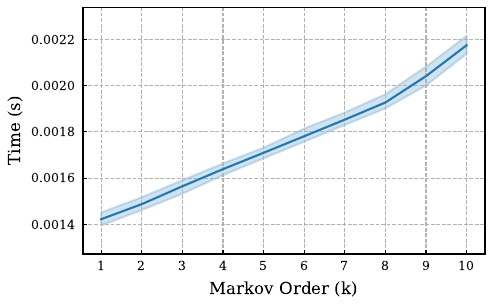}
    }
    \hfill
    \subfloat[Change in runtime with increasing number of states.\label{fig:eq_mt}]{%
        \includegraphics[width=0.45\linewidth]{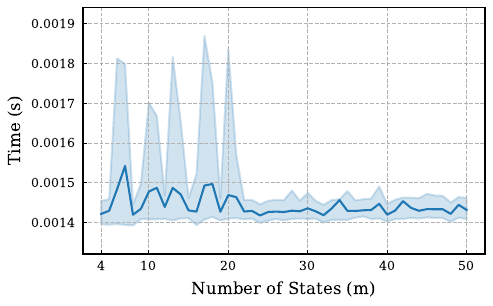}
    }
    \caption{
        Runtimes (in seconds) of the proposed update equation to process $10^3$ observations.
        Each value is the average of $100$ runs.
    }
    \label{fig:eq}
\end{figure}

\begin{figure}[htbp]
    \centering
    \subfloat[Change in runtime with increasing Markov order.\label{fig:kt_e2e}]{%
        \includegraphics[width=0.45\linewidth]{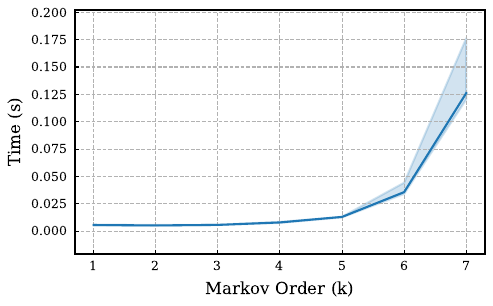}
    }
    \hfill
    \subfloat[Change in runtime with increasing number of states.\label{fig:km_e2e}]{%
        \includegraphics[width=0.45\linewidth]{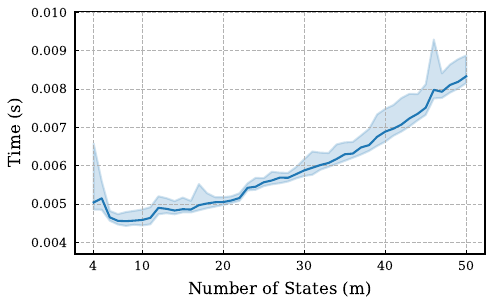}
    }
    \caption{
        Runtimes (in seconds) of end-to-end EMC to process $10^3$ observations.
        Each value is the average of $100$ runs.
    }
    \label{fig:e2e}
\end{figure}

\subsection{Supplemental Tables}

\setlength\aboverulesep{0pt}
\setlength\belowrulesep{0pt}
\setlength\cellspacetoplimit{3pt}
\setlength\cellspacebottomlimit{3pt}

\begin{table}[htbp]
    \centering
    \renewcommand{\arraystretch}{1.5}
    \rowcolors{2}{rowColor}{white}
    \begin{tabularx}{\textwidth}{lX}
        \toprule
        \rowcolor{hdrColor}
        \textbf{Algorithm} & \textbf{Hyperparameter Values} \\ \midrule
        EMC &
        $\lambda=\{0.94,0.95\}$,
        $\beta=0$,
        $\delta=\{0.3,0.2\}$,
        $\eta=\{0.5,0.25\}$,
        $\tau=100$ \\
        MC-ADWIN &
        $\delta=0.002$,
        $\text{clock}=3$,
        $\text{max\_buckets}=34$,
        $\text{min\_window\_length}=14$,
        $\text{grace\_period}=58$ \\
        MC-SW &
        $w=\{100,328,500\}$ \\
        \bottomrule
    \end{tabularx}
    \caption{
        Hyperparameter values for each algorithm run in \emph{probability tracking} experiments on synthetic data.
        For EMC, parameters shown as tuples represent the values for fast and slow learning, respectively.
        The Markov order $k$ is $1$ for all algorithms.
    }
    \label{tab:params_pt}
\end{table}

\begin{table}[htbp]
    \centering
    \renewcommand{\arraystretch}{1.5}
    \rowcolors{2}{rowColor}{white}
    \begin{tabulary}{\textwidth}{lL}
        \toprule
        \rowcolor{hdrColor}
        \textbf{Algorithm} & \textbf{Hyperparameter Values} \\ \midrule
        ADWIN &
        $\delta=0.017$,
        $\text{clock}=21$,
        $\text{max\_buckets}=36$,
        $\text{min\_window\_length}=28$,
        $\text{grace\_period}=79$ \\
        EMC &
        $\lambda=\{0.91,0.95\}$,
        $\beta=0$,
        $\delta=\{0.3,0.05\}$,
        $\eta=\{0.35,0.35\}$,
        $\tau=75$ \\
        KSWIN &
        $\alpha=0.005$,
        $n=811$,
        $r=211$ \\
        PHT &
        $\delta=0.034$,
        $\lambda=53.51$,
        $\alpha=0.993$,
        $\text{min\_instances}=230$ \\
        SCD &
        $\lambda=0.98$,
        $L=100$,
        $N=4$,
        $M=100$,
        $\text{minHeight}=0.1$,
        $\text{minPD}=500$ \\
        \bottomrule
    \end{tabulary}
    \caption{
        Hyperparameter values for each algorithm run in \emph{change point detection} experiments on synthetic data.
        For EMC, parameters shown as tuples represent the values for fast and slow learning, respectively.
        The Markov order $k$ is $1$ for all algorithms.
    }
    \label{tab:params_cpd}
\end{table}

\begin{table}[htbp]
    \centering
    \renewcommand{\arraystretch}{1.5}
    \rowcolors{2}{rowColor}{white}
    \begin{tabulary}{\textwidth}{lL}
        \toprule
        \rowcolor{hdrColor}
        \textbf{Algorithm} & \textbf{Hyperparameter Values} \\ \midrule
        EMC &
        $\lambda=\{0.92,0.97\}$,
        $\beta=0$,
        $\delta=\{0.2,0.05\}$,
        $\eta=\{0.35,0.3\}$,
        $\tau=25$ \\
        EPSTM &
        $t=2$,
        $L=1$,
        $\text{distance\_threshold}=0.1$ (for clustering) \\
        CluStream &
        $\text{n\_macro\_clusters}=6$,
        $\text{time\_gap}=129$,
        $\text{max\_micro\_clusters}=72$,
        $\text{time\_window}=1299$,
        $\text{micro\_cluster\_r\_factor}=4$ \\
        DBStream &
        $\text{clustering\_threshold}=1.12$,
        $\text{fading\_factor}=0.007$,
        $\text{cleanup\_interval}=13$,
        $\text{intersection\_factor}=0.16$,  
        $\text{minimum\_weight}=6.56$ \\
        \bottomrule
    \end{tabulary}
    \caption{
        Hyperparameter values for each algorithm run in \emph{mode discovery and recognition} experiments on synthetic data.
        For EMC, parameters shown as tuples represent the values for fast and slow learning, respectively.
        The Markov order $k$ is $1$ for all algorithms.
    }
    \label{tab:params_mdr}
\end{table}

\begin{table}[htbp]
    \centering
    \renewcommand{\arraystretch}{1.5}
    \rowcolors{2}{rowColor}{white}
    \begin{tabularx}{\textwidth}{Xcccccc}
        \toprule
        \rowcolor{hdrColor}
        \textbf{Application} & $\bm{k}$ & $\bm{\lambda}$ & $\bm{\beta}$ & $\bm{\delta}$ & $\bm{\eta}$ & $\bm{\tau}$ \\ \midrule
        Human Activity Recognition &
        $1$ &
        $\{0.94,0.95\}$ &
        $0.01$ &
        $\{0.4,0.05\}$ &
        $\{0.2,0.35\}$ &
        $50$ \\
        Electric Motor Condition Monitoring &
        $2$ &
        $\{0.94,0.96\}$ &
        $0.003$ &
        $\{0.15,0.085\}$ &
        $\{0.07,0.2\}$ &
        $25$ \\
        EEG Eye State Detection &
        $1$ &
        $\{0.93,0.96\}$ &
        $0.001$ &
        $\{0.3,0.15\}$ &
        $\{0.2,0.45\}$ &
        $25$ \\
        \bottomrule
    \end{tabularx}
    \caption{
        Hyperparameter values for EMC on applications on real-world data.
        Parameters shown as tuples represent the values for fast and slow learning, respectively.
    }
    \label{tab:params_har}
\end{table}

\end{document}